%% file: jdot.tex
\documentclass{article}

\pdfoutput=1

\usepackage[final,nonatbib]{nips_2017}

\usepackage[utf8]{inputenc} 
\usepackage[T1]{fontenc}    
\usepackage{hyperref}       
\usepackage{url}            
\usepackage{booktabs}       
\usepackage{amsfonts}       
\usepackage{nicefrac}       
\usepackage{microtype}      

\usepackage{graphicx} 

\usepackage{algorithm}
\usepackage{algorithmic}
\usepackage{lipsum}

\usepackage[dvipsnames]{xcolor}

\usepackage{amssymb,amsmath}
\input{notations}

\title{Joint distribution optimal transportation for domain adaptation}

\author{
  Nicolas Courty\thanks{Both authors contributed equally.} \\
 Universit{\'e} de Bretagne Sud,\\
  IRISA, UMR 6074, CNRS,\\
  \texttt{courty@univ-ubs.fr} \\
  \And
    R{\'e}mi Flamary$^*$\\
 Universit{\'e} Côte d’Azur,\\
  Lagrange, UMR 7293 , CNRS, OCA\\
  \texttt{remi.flamary@unice.fr} \\
  \And
  Amaury Habrard\\
Univ Lyon, UJM-Saint-Etienne, CNRS,\\
Lab. Hubert Curien UMR 5516, F-42023\\
 \texttt{amaury.habrard@univ-st-etienne.fr}\\
  \And
  Alain Rakotomamonjy\\
Normandie Universite \\
 Universit{\'e} de Rouen, LITIS EA 4108\\
 \texttt{alain.rakoto@insa-rouen.fr}\\ 
}

\begin{document} 
\maketitle

\begin{abstract} 
This paper deals with the unsupervised domain adaptation problem, where one wants to estimate a prediction function $f$ in a given target domain without any labeled sample by exploiting the knowledge available from a source domain where labels are known. Our work makes the following assumption: there exists a  non-linear transformation between the joint feature/label space distributions of the two domain $\ps$ and $\pt$ that can be estimated with optimal transport. We propose a solution of this problem 
that allows to recover an estimated target $\pt^f=(X,f(X))$ by optimizing simultaneously the optimal coupling and $f$. We show that our method corresponds to the minimization of a bound on the target error, and provide an efficient algorithmic solution, for which convergence is proved. The versatility of our approach, both in terms of class of hypothesis 
or loss functions is demonstrated with real world classification and regression problems, for which we reach or surpass state-of-the-art results.

\end{abstract} 

\section{Introduction}
\label{sec:intro}

\input{intro.tex}

\section{Joint distribution Optimal Transport}
\label{sec:method}

\input{method.tex}

\section{A Bound on the Target Error}
\label{sec:theory}

\input{theory.tex}

\section{Learning with Joint Distribution OT}
\label{sec:algo}

\input{algorithm.tex}

\section{Numerical experiments}
\label{sec:expe}

\input{expe.tex}

\section{Discussion and conclusion}
\label{sec:conc}
We have presented in this paper the Joint Distribution Optimal Transport for domain adaptation, which is a principled way of performing domain adaptation with optimal transport. JDOT assumes the existence of a transfer map that transforms a source domain joint distribution $\ps(X,Y)$ into a target domain equivalent version $\pt(X,Y)$. Through this transformation, the alignment of both feature space and conditional distributions is operated, allowing to devise an efficient algorithm that simultaneously optimizes for a coupling between $\ps$ and $\pt$ and a prediction function that solves the transfer problem.
We also proved that learning with {\bf JDOT} is equivalent to minimizing a  bound on the target distribution.
We have demonstrated through experiments on classical real-world benchmark datasets the superiority of our approach w.r.t. several state-of-the-art methods, including previous work on optimal transport based domain adaptation, domain adversarial neural networks or transfer components, on a variety of task including classification and regression. We have also showed the versatility of our method, that can accommodate with several types of loss functions (mse, hinge) or class of hypothesis (including kernel machines or neural networks). 
Potential follow-ups of this work include a semi-supervised extension (using unlabelled examples in source domain)  and investigating stochastic techniques for solving efficiently the adaptation.
From a theoretical standpoint, future works include a deeper study of  probabilistic transfer lipschitzness and the development of guarantees able to take into the complexity of the hypothesis class and the space of possible transport plans.




\section*{Acknowledgements}
This work benefited from the support of the project OATMIL ANR-17-CE23-0012 of the French National Research Agency (ANR), the Normandie Projet GRR-DAISI, European funding FEDER DAISI and CNRS funding from the D\'{e}fi Imag'In. The authors also wish to thank Kai Zhang and Qiaojun Wang for providing the Wifi localization dataset.

\bibliographystyle{unsrt}
\begin{small}
\bibliography{biblio,bibliosup}
\end{small}

\newpage
\appendix{Appendix Section}

\input{supplementary_full.tex}

\end{document}

%% file: notations.tex
\newtheorem{theorem}{Theorem}[section]

\newenvironment{proof}[1][Proof]{\begin{trivlist}
\item[\hskip \labelsep {\bfseries #1}]}{\end{trivlist}}
\newenvironment{definition}[1][Definition]{\begin{trivlist}
\item[\hskip \labelsep {\bfseries #1}]}{\end{trivlist}}

\newcommand{\qed}{\nobreak \ifvmode \relax \else
      \ifdim\lastskip<1.5em \hskip-\lastskip
      \hskip1.5em plus0em minus0.5em \fi \nobreak
      \vrule height0.75em width0.5em depth0.25em\fi}

\def\argmin{\mathop{\mathrm{argmin}}}

\def\C{{\bf C}}

\newcommand{\x}{{\bf x}}
\newcommand{\z}{{\bf z}}
\newcommand{\w}{{\bf w}}

\renewcommand{\xi}{{\bf x}_i}
\newcommand{\xsi}{{\bf x}^s_i}
\newcommand{\xti}{{\bf x}^t_i}
\newcommand{\ysi}{{\bf y}^s_i}

\newcommand{\xtj}{{\bf x}^t_j}

\newcommand{\T}{{\mathcal{T}}}

\newcommand{\Tmus}{{\T}{\#\mu_s}}

\newcommand{\Gzero}{{\boldsymbol{\gamma}}_{0}}
\newcommand{\G}{{\boldsymbol{\gamma}}}

\newcommand{\dist}[1]{\text{d}(#1)}

\newcommand{\Closs}[1]{\mathcal{D}(#1)}

\newcommand{\prob}[1]{\mathcal{P}(#1)}

\newcommand{\proba}[1]{\mathcal{P}(#1)}
\newcommand{\probas}[1]{\mathcal{P}_s(#1)}
\newcommand{\probat}[1]{\mathcal{P}_t(#1)}

\def\ps{{\mathcal{P}_s}}
\def\pt{\mathcal{P}_t}

\newcommand{\ft}[1]{f(#1)}

\newcommand{\Xt}{{\bf X}_t}
\newcommand{\Xs}{{\bf X}_s}
\newcommand{\Yt}{{\bf Y}_t}
\newcommand{\Ys}{{\bf Y}_s}

\def\R{\mathbb{R}}

\def\C{\mathbf{C}}

\def\x{\mathbf{x}}

\def\d{\mathbf{d}}

\def\ie{{\em i.e.}}

\def\err{\text{err}}

\DeclareMathOperator*{\E}{\mathbb{E}}

\def\matP{{\mathbf{P}}}

\def\loss{\mathcal{L}}

\def\equaldef{\stackrel{\text{def}}{=}}

\newtheorem{proposition*}{\textbf{Proposition}}
\newtheorem{theorem*}{\textbf{Theorem}}
\newtheorem{Lemma*}{\textbf{Lemma}}
\newtheorem{remark*}{\textbf{Remark}}
\newtheorem{definition*}{\textbf{Definition}}


%% file: intro.tex

In the context of supervised learning, one generally assumes that the test data is a realization of the same process
that generated the learning set. Yet, in many practical applications it is often not the case, since several factors can slightly alter this process.
The particular case of visual adaptation~\cite{Pat14} in computer vision is a good example: given a new dataset of images without 
any label, one may want to exploit a different annotated dataset, provided that 
they share sufficient common information and labels. However, the generating process can be different in several aspects, such as 
the conditions and devices used for acquisition, different pre-processing, different compressions, etc. Domain adaptation 
techniques aim at alleviating this issue by transferring knowledge between domains~\cite{Pan10}. We propose in this paper a 
principled and theoretically founded way of tackling this problem. 
    
The domain adaptation (DA) problem is not new and has received a lot of attention during the past ten years. State-of-the-art methods are
mainly differing by the assumptions made over the change in data distributions. In the {\em covariate shift} assumption, the differences
between the domains are characterized by a change in the feature distributions $\prob{X}$, while the conditional distributions $\prob{Y|X}$
remain unchanged ($X$ and $Y$ being respectively the instance and label spaces). Importance re-weighting can be used to learn a new classifier (e.g.~\cite{Sug08}), provided that the overlapping of the 
distributions is large enough. Kernel alignment~\cite{Zhang13} has also been considered for the same purpose. Other types of method, denoted 
as {\em Invariant Components} by Gong and co-authors~\cite{Gong16}, are looking for a transformation $\T$  such that the new representations 
of input data are matching, {\em i.e.} $\probas{\T(X)}=\probat{\T(X)}$. Methods are then differing by: {\em i)} The considered class of transformation, 
that are generally defined as projections (e.g.~\cite{gong12,Baktashmotlagh13,Fernando13,long2014,Gong16}), affine transform~\cite{Zhang13}
or non-linear transformation as expressed by neural networks~\cite{ganin2015,ganin2016} 
{\em ii)} The types of divergences used to compare $\probas{\T(X)}$ and $\probat{\T(X)}$, such as Kullback Leibler~\cite{Si10} or {\em Maximum 
Mean Discrepancy}~\cite{long2014,Gong16}. Those divergences usually require that the distributions share a common support to be defined. 
A particular case is found in the use of optimal transport, introduced for domain adaptation by~\cite{courty14,courty2016}. 
$\T$ is then defined to be a push-forward operator such that $\probas{X}=\probat{\T(X)}$ and that minimizes 
a global transportation effort or cost between 
distributions. The associated divergence is the so-called Wasserstein metric, that has a natural Lagrangian formulation and avoids the estimation
of continuous distribution by means of kernel. As such, it also alleviates the need for a shared support. 

The methods discussed above implicitly assume that the conditional distributions are  unchanged by $\T$, {\em i.e.} $\probas{Y|\T(X)}\approx\probat{Y|\T(X)}$
but there is no clear reason for this assumption to hold. A more general approach is to adapt both marginal feature and conditional distributions by minimizing a global divergence between them. However, 
this task is usually hard since no label is available in the target domain and therefore no empirical version $\probat{Y|X}$ can be used. This was achieved by 
restricting to specific class of transformation such as projection~\cite{long2014,Gong16}. 

\textbf{Contributions and outline.}
In this work we propose a novel framework for unsupervised domain adaptation between joint distributions. We propose to find a function $f$ that predicts an output value
given an input $\x \in \mathcal{X}$, and
                                  that minimizes the optimal transport loss between the joint source distribution $\ps$ and an estimated target joint distribution $\pt^f=(X,f(X))$ depending on $f$ (detailed in Section \ref{sec:method}). The method is denoted as JDOT for ``Joint Distribution Optimal Transport" in the remainder. We show that the resulting optimization problem stands for a minimization of a bound  on the target error of $f$ (Section~\ref{sec:theory})  and propose an efficient algorithm to solve it (Section~\ref{sec:algo}). Our approach is very general and does not require to learn explicitly a transformation, as it directly solves for the best function. We show that it can handle both regression and classification problems with a large class of functions $f$ including kernel machines and neural networks. 
We finally provide several numerical experiments on real regression and classification problems that show the performances of JDOT over the state-of-the-art (Section~\ref{sec:expe}). 




%% file: method.tex
Let $\Omega\in\R^d$ be a compact input measurable space of dimension $d$ and $\mathcal{C}$ the set of labels. $\prob{\Omega}$ denotes the set of all the probability measures over $\Omega$. The standard learning paradigm assumes classically the existence of a set of data  $\Xs = \{ \xsi \}_{i=1}^{N_s}$ associated with a set of class label information $\Ys = \{\ysi\}_{i=1}^{N_s}$, $\ysi \in \mathcal{C}$ (the learning set), and a data set with unknown labels  $\Xt = \{  \xti \}_{i=1}^{N_t}$ (the testing set). In order to determine the set of labels $\Yt$  associated with $\Xt$ , one usually relies on an empirical estimate of the joint probability distribution $\proba{X,Y} \in \prob{\Omega\times\mathcal{C}}$ from $(\Xs,\Ys)$, and the assumption that $\Xs$ and $\Xt$ are drawn from the same distribution $\mu \in \prob{\Omega}$.  In the considered adaptation problem, one assumes the existence of two distinct joint probability distributions $\probas{X,Y}$ and $\probat{X,Y}$ which correspond respectively to two different {\em source} and {\em target} domains. We will write  $\mu_s$ and $\mu_t$ their respective marginal distributions over $X$. 
 
\subsection{Optimal transport in domain adaptation}
The Monge problem is seeking for a map $\T_0:\Omega\rightarrow\Omega$ that pushes $\mu_s$ toward $\mu_t$ defined as:
\begin{equation*}
 \T_0 = \argmin_\T  \int_{\Omega} \dist{\x,\T(\x)}d\mu_s(\x), \;\;\;\;\; \text{s.t.  } \; \Tmus=\mu_t, 
\end{equation*}
where  $\Tmus$ the {\em image measure} of $\mu_s$ by $\T$, verifying:
\begin{equation}
 \Tmus(A) = \mu_t(\T^{-1}(A)),\;\;\forall\text{ Borel subset $A \subset \Omega$},
\end{equation}
and $\text{d}:\Omega \times \Omega \rightarrow \mathbb{R}^+$ is a metric. In the remainder, we will always consider 
without further notification the case where $\text{d}$ is the squared Euclidean metric. When $\T_0$ exists, it is called an optimal transport map,
but it is not always the case ({\em e.g.} assume that $\mu_s$ is defined by one Dirac measure and $\mu_t$ by two). A 
relaxed version of this problem has been proposed by Kantorovitch~\cite{Kantorovich42}, who rather seeks for a transport plan 
(or equivalently a joint probability distribution) $\G \in \proba{\Omega \times \Omega}$ such that:
\begin{equation}
 \Gzero = \argmin_{\G \in \Pi(\mu_s,\mu_t)}  \int_{\Omega\times\Omega} d(\x_1,\x_2)d\G(\x_1,\x_2),
 \label{eq:kanto}
\end{equation}
where  $\Pi(\mu_s,\mu_t)=\{\G \in \proba{\Omega \times \Omega} | p^+\#\G=\mu_s, p^-\#\G=\mu_t\}$ and $p^+$ and $p^-$ denotes the two marginal projections of 
$\Omega \times \Omega$ to $\Omega$. Minimizers of this problem are called optimal transport plans. Should $\Gzero$ be of the form $(id \times \T)\#\mu_s$, then 
the solution to Kantorovich and Monge problems coincide. As such the Kantorovich relaxation can be seen as a generalization of the Monge problem, with less constraints on 
the existence and uniqueness of solutions~\cite{santambrogio2015}. 

Optimal transport has been used in DA as a principled way to bring the source and target distribution closer~\cite{courty14,courty2016,perrot2016}, by 
seeking for a transport plan between the empirical distributions of $\Xs$ and $\Xt$ and interpolating $\Xs$ thanks to a barycentric mapping~\cite{courty2016},
or by estimating a mapping which is not the solution of Monge problem but allows to map unseen samples~\cite{perrot2016}. Moreover, they 
show that better constraining the structure of $\G$ through entropic or classwise regularization terms helps in achieving better empirical results. 

\subsection{Joint distribution optimal transport loss}
The main idea of this work is is to handle a change in both marginal and conditional distributions. As such, we are looking for a transformation $\T$ that will align directly the joint distributions $\ps$ and $\pt$. 
Following the Kantovorich formulation of \eqref{eq:kanto}, $\T$ will be implicitly expressed through a coupling between both joint distributions as: 
\begin{equation}
 \Gzero = \argmin_{\G \in \Pi(\ps,\pt)}  \int_{(\Omega\times\mathcal{C})^2} \Closs{\x_1,y_1;\x_2,y_2}d\G(\x_1,y_1;\x_2,y_2),
 \label{eq:tloss}
\end{equation}
where $\Closs{\x_1,y_1;\x_2,y_2}=\alpha \dist{\x_1,\x_2}+\mathcal{L}(y_1,y_2)$
is a joint cost measure combining both the distances between the samples and a loss function $\mathcal{L}$ measuring the discrepancy between $y_1$ and $y_2$.
While this joint cost is specific (separable), we leave for future work the analysis of generic joint cost function.
Putting it in words, matching close source and target samples with similar labels costs few. $\alpha$ is a positive parameter which balances the  metric in the feature space and the loss. 
As such, when $\alpha \rightarrow +\infty$, this cost is dominated by the metric in the input feature space, and the solution of the coupling problem is the same 
as in~\cite{courty2016}. 
It can be shown that a minimizer to (\ref{eq:tloss}) always exists and is unique provided that $\Closs{\cdot}$ is lower semi-continuous (see~\cite{villani09}, Theorem 4.1),
which is the case when $\dist{\cdot}$ is a norm and  for every usual loss functions~\cite{rosasco2004}.

In the unsupervised DA problem, one does not have access to labels in the target domain, and as such it is not possible to find the optimal 
coupling. Since our goal is to find a function on the target domain ${f}:\Omega\rightarrow\mathcal{C}$, we suggest to replace $y_2$ by a proxy 
$\ft{\x_2}$. This leads to the definition of the following
joint distribution that uses a given function $f$ as a proxy for $y$:
\begin{equation}
\pt^f=(\x,f(\x))_{\x\sim \mu_t} \label{eq:proxy}
\end{equation}
In practice we consider empirical versions of $\ps$ and $\pt^f$, {\em i.e.} $\hat{\ps} = \frac{1}{N_s}\sum_{i=1}^{N_s} \delta_{\xsi,\ysi}$ 
and $\hat{\pt^f}  = \frac{1}{N_t}\sum_{i=1}^{N_t} \delta_{\xti,\ft{\xti}}$. $\G$ is then a matrix which belongs to $\Delta$ , \ie the transportation polytope of non-negative matrices between uniform distributions. 
Since our goal is to estimate a  prediction $f$ on the target domain, 
we propose to find the one that produces predictions that match optimally source labels to the 
aligned target instances in the transport plan. For this purpose, we propose to solve the following problem for JDOT:
\begin{equation}
 \min_{f,\G \in \Delta}  \sum_{ij} \Closs{\xsi,\ysi;\xtj,f({\xtj})}\G_{ij} \quad \equiv \quad  \min_{f} W_1(\hat\ps,\hat{\pt^f})  
 \label{eq:tloss2}
\end{equation} 
where $W_1$ is the 1-Wasserstein distance for the loss $\Closs{\x_1,y_1;\x_2,y_2}=\alpha \dist{\x_1,\x_2}+\mathcal{L}(y_1,y_2)$.
We will make clear in the next section that the function $f$ we retrieve is theoretically sound with respect to the target error.
Note that in practice we add a regularization term for function $f$ in order to avoid overfitting as discussed in Section~\ref{sec:algo}.
An illustration of JDOT for a regression problem is given in Figure \ref{fig:tloss_reg}. In this figure, we have very different joint and marginal distributions but we want to illustrate that the OT matrix $\G$ obtained using the true empirical distribution $\pt$ is very similar to the one obtained with the proxy $\pt^f$ which leads to a very good model for JDOT.

 \begin{figure}[!t]
\includegraphics[width=1\linewidth]{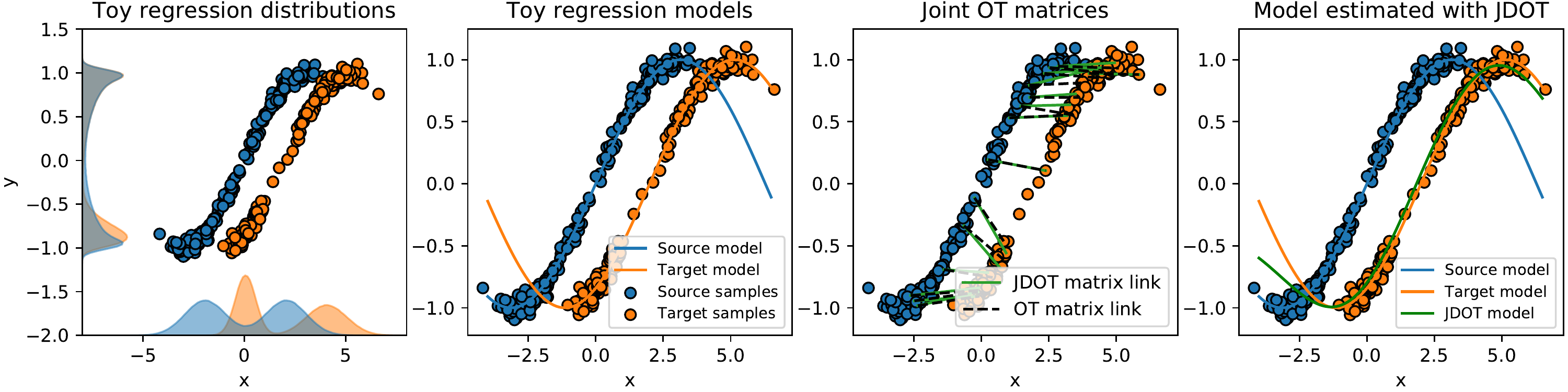}\vspace{-5mm}
\caption{Illustration of JDOT on a 1D regression problem. (left) Source and target empirical distributions and marginals (middle left) Source and target models (middle right) OT matrix on empirical joint distributions and with JDOT proxy joint distribution (right) estimated prediction function $f$. } 
  \label{fig:tloss_reg}\vspace{-3mm}
\end{figure} 


\textbf{Choice of $\alpha$.} This is an important parameter balancing the alignment of feature space and labels. 
A natural choice of the $\alpha$ parameter is obtained by normalizing the range of values of $\dist{\xsi,\xtj}$ with $\alpha=1/\max_{i,j}  \dist{\xsi,\xtj}$. In the numerical experiment section, we show that
this setting is very good in two out of three experiments. However, in some cases, better performances are obtained with a cross-validation of this parameter. 
Also note that $\alpha$ is strongly linked to the smoothness of the  loss $\mathcal{L}$ and of the optimal labelling functions and can be seen as a Lipschitz constant in the bound of Theorem \ref{thm:bound}.

\textbf{Relation to other optimal transport based DA methods.} Previous DA methods based on optimal transport~\cite{courty2016,perrot2016} do not not only differ by 
the nature of the considered distributions, but also in the way the optimal plan is used to find $f$. They learn a complex mapping  between the source and target distributions when the objective is only to estimate a prediction function $f$ on target. 
To do so, they rely on a barycentric mapping that minimizes only approximately the Wasserstein distance between the distributions. As discussed in Section~\ref{sec:algo}, JDOT uses the optimal plan to propagate and fuse the labels from the source to target. Not only are the performances enhanced, but we also show how this approach is more theoretically well grounded in next section \ref{sec:theory}.


\textbf{Relation to Transport $L^p$ distances.} Recently, Thorpe and co-authors introduced the Transportation $\text{L}^p$ distance~\cite{Thorpe16}. Their objective is to compute a meaningful distance between multi-dimensional signals. Interestingly their distance can be seen as optimal transport between two distributions of the form \eqref{eq:proxy} where the functions are known  and the label loss $\mathcal{L}$ is chosen as a $L^p$ distance. While their approach is inspirational, JDOT is different both in its formulation, where we introduce a more general class of loss $\mathcal{L}$, and in its objective, as our goal is to estimate the target function $f$ which is not known \emph{a priori}. Finally we show theoretically and empirically that our formulation addresses successfully the problem of domain adaptation. 



%% file: theory.tex

Let $f$ be an hypothesis function from a given class of hypothesis $\mathcal{H}$. We define the expected loss in the target domain $err_T(f)$ as
$err_T(f)\equaldef  \E_{(\x,y)\sim\pt} \loss(y,f(\x))$.
We define similarly $err_S(f)$ for the source domain. 
We assume the loss function  $\loss$ to be bounded, symmetric, $k$-lipschitz and satisfying the triangle inequality. 


To provide some guarantees on our method, we consider an adaptation of the notion probabilistic Lipschitzness introduced in \cite{UrnerSB11,Ben-DavidSU12} which assumes that two close instances must have the same labels with high probability. It corresponds to a relaxation of the classic Lipschitzness allowing one to model the marginal-label relatedness such as in Nearest-Neighbor classification, linear classification or cluster assumption.  We propose an extension of this notion in a domain adaptation context by assuming that a labeling function must comply with two close instances of each domain w.r.t. a coupling $\Pi$.
\begin{definition} ({\bf Probabilistic Transfer Lipschitzness})
Let $\mu_s$ and $\mu_t$ be respectively the source and target distributions. Let  $\phi:\mathbb{R}\rightarrow
[0,1]$.  A labeling function $f:\Omega\rightarrow \mathbb{R}$ and a joint distribution $\Pi(\mu_s,\mu_t)$ over  $\mu_s$ and $\mu_t$ are
$\phi$-Lipschitz transferable  if
for all $\lambda>0$: 
\begin{equation*}
Pr_{(\x_1,\x_2)\sim \Pi(\mu_s,\mu_t)}\left[ |f(\x_1)-f(\x_2)|>{\lambda} d(\x_1,\x_2) 
\right]\leq \phi(\lambda).
\end{equation*}
\end{definition}
Intuitively, given a deterministic labeling functions $f$ and a coupling $\Pi$, it bounds the probability of finding pairs of source-target instances labelled differently in a $(1/\lambda)$-ball  with respect to $\Pi$.



We can now give our main result (simplified version):
\begin{theorem} \label{thm:bound}
Let $f$ be any labeling function of $\in\mathcal{H}$. Let $\Pi^*=\argmin_{\Pi\in\Pi(\ps,\pt^f)}\int_{(\Omega\times  \mathcal{C})^2} \alpha d(\x_s,\x_t)+\loss(y_s,y_t)d\Pi(\x_s,y_s;\x_t,y_t)$ and $W_1(\hat{\ps},\hat{\pt^f})$ the associated 1-Wasserstein distance. Let $f^*\in\mathcal{H}$  be a Lipschitz labeling function that verifies the $\phi$-probabilistic transfer Lipschitzness (PTL) assumption w.r.t. $\Pi^*$ and that minimizes the joint error $err_S(f^*)+err_T(f^*)$ w.r.t all PTL functions compatible with $\Pi^*$. We assume the input instances are bounded s.t. $|f^*(\x_1)-f^*(\x_2)|\leq M$ for all $\x_1,\x_2$. Let $\loss$ be any symmetric loss function, $k$-Lipschitz and  satisfying the triangle inequality. 
Consider a sample of $N_s$ labeled source instances drawn from
$\ps$ and $N_t$ unlabeled instances drawn from $\mu_t$, and then  for
all $\lambda>0$, with $\alpha=k\lambda$, we have with probability at least $1-\delta$ that:
\begin{eqnarray*}
\err_T(f)\!\!&\!\!\leq\!\!&\!\! \small W_1(\hat{\ps},\hat{\pt^f})+\sqrt{\frac{2}{c'}\log(\frac{2}{\delta})}\left(\frac{1}{\sqrt{N_S}}+\frac{1}{\sqrt{N_T}}\right)
+err_S(f^*)+err_T(f^*)
+kM\phi(\lambda).
\end{eqnarray*}
\end{theorem}

The detailed proof of Theorem \ref{thm:bound} is given in the supplementary material. The previous bound on the target error above is interesting to interpret. The first two terms correspond to the objective function \eqref{eq:tloss2} we propose to minimize accompanied with a sampling bound. The last term $\phi(\lambda)$ assesses the probability under which the probabilistic Lipschitzness does not hold. The remaining two terms involving $f^*$ correspond to the joint error minimizer illustrating that domain adaptation can work only if we can predict well in both domains, similarly to existing results in the literature \cite{Mansour09,Ben-DavidMLJ10}.  If the last  terms are small enough, adaptation is possible if we are able to align well $\ps$ and $\pt^f$, provided that $f^*$ and $\Pi^*$ verify the PTL.  Finally, note that $\alpha=k\lambda$ and tuning this parameter is thus actually related to finding the Lipschitz constants of the problem.



%% file: algorithm.tex

In this section, we provide some details about  the JDOT's optimization problem given in Equation \eqref{eq:tloss2} and discuss algorithms for its resolution.
We will assume that the function space $\mathcal{H}$ to which $f$ belongs
is either a RKHS or a function space parametrized by some parameters $\w \in \R^p$. 
This framework encompasses linear models, neural networks, and kernel methods.
Accordingly, we are going to define a regularization term $\Omega(f)$ on $f$.
Depending on how $\mathcal{H}$ is defined,  $\Omega(f)$ is either a non-decreasing function of the squared-norm induced by the RKHS (so that the representer theorem  is applicable) or a squared-norm  on the vector parameter.  We will further assume that $\Omega(f)$ is continuously differentiable.
As discussed above, $f$ is to be learned according to the following optimization problem
\begin{equation} \small
  \min_{f \in \mathcal{H},\G\in\Delta} \sum_{i,j} \G_{i,j} \left(\alpha d(\xsi,\x_j^t)+\mathcal{L}(y^s_i,f(\x_j^t))\right)+\lambda \Omega(f)
  \label{eq:learnprob}
\end{equation}
where the loss function $\mathcal{L}$ is continuous and differentiable with respects to its second variable.
Note that while the above problem does not involve any regularization term on the coupling matrix $\G$, it is essentially for the sake of simplicity and readability. Regularizers like entropic regularization \cite{CuturiSinkhorn}, which is
relevant when the number of samples is very large, can still be used without significant change to the algorithmic framework.

%
%

\textbf{Optimization procedure.}
According to the above hypotheses on $f$ and $\mathcal{L}$, Problem \eqref{eq:learnprob} is smooth and the constraints are separable
according to $f$ and $\G$. Hence, a natural way to solve the problem \eqref{eq:learnprob} is to rely on alternate optimization w.r.t. both parameters $\G$ and $f$. This algorithm well-known as Block Coordinate Descent (BCD) or Gauss-Seidel method (the pseudo code of the algorithm is given in appendix). 
Block optimization steps are discussed with further  details in the following. 


Solving with fixed $f$ boils down to 
a classical OT problem with a loss matrix $\C$ such that $C_{i,j}=\alpha d(\xsi,\x_j^t)+\mathcal{L}(y^s_i,f(\x_j^t))$. We can use classical OT solvers such as the network simplex algorithm, but 
other strategies can be considered, such as regularized OT \cite{CuturiSinkhorn} or stochastic versions~\cite{genevay2016}.



The optimization problem with fixed $\G$ leads to a new learning problem expressed as
\begin{equation}
  \min_{f \in\mathcal{H}} \quad \sum_{i,j} \G_{i,j} \mathcal{L}(y^s_i,f(\x_j^t))+\lambda \Omega(f)
  \label{eq:learn_f}
\end{equation}
Note how the data fitting term   elegantly and naturally encodes the transfer of  source labels $y^s_i$ through estimated labels of test
samples with a weighting depending on the optimal transport matrix.
However, this comes at the price of having a quadratic number $N_s N_t$ of terms, which can be considered as computationally expensive.
%
We will see in the sequel that we can benefit from the structure of the chosen loss to greatly reduce its complexity. 
In addition, we emphasize that when $\mathcal{H}$ is a RKHS,  owing to 
kernel trick and the representer theorem, problem \eqref{eq:learn_f}
can be re-expressed as an optimization problem with $N_t$ number of parameters all belonging to $\R$.

Let us now discuss  briefly the convergence of the proposed algorithm. 
Owing to the $2$-block  coordinate descent structure, to the differentiability of the objective function in Problem  \eqref{eq:learnprob} and constraints on 
$f$ (or its kernel trick parameters) and $\G$ are closed, non-empty and convex, convergence result of Grippo et al. \cite{grippo2000convergence} on 2-block Gauss-Seidel methods
directly applies. It states that if the sequence $\{\G^k, f^k\}$ produced
by the algorithm has limit points then every limit point of the sequence
is a critical point of Problem \eqref{eq:learnprob}.

\textbf{Estimating $f$ for least square regression problems.}
We detail the use of JDOT for transfer least-square regression problem i.e when $\mathcal{L}$ is the squared-loss. In this context, when the optimal transport matrix $\G$ is fixed  the learning problem boils down to
\begin{equation}
  \min_{f\in\mathcal{H}}  \quad \sum_{j} \frac{1}{n_t}\|\hat y_j-f(\x_j^t)\|^2+\lambda \|f\|^2
  \label{eq:learn_reg2}
\end{equation}
where the  $\hat y_j=n_t\sum_j \G_{i,j} y^s_i$ is a weighted average of the source target values. 
Note that this simplification results from the properties of the quadratic loss and that it may not occur for more complex regression
loss.

\textbf{Estimating $f$ for hinge loss classification problems.}
We now aim at estimating a multiclass classifier with a one-against-all strategy. We suppose that the data fitting is the binary squared hinge loss of the form $\mathcal{L}(y,f(\x))=\max(0,1-yf(\x))^2$. In a One-Against-All strategy we often use the binary matrices $\matP$ such that  $P^s_{i,k}=1$ if sample $i$ is of class $k$ else  $P^s_{i,k}=0$.
Denote as $f_k \in \mathcal{H}$ the decision function related to the $k$-vs-all problem. The learning problem \eqref{eq:learn_f} can now  be expressed as
\begin{equation}
  \min_{f_k\in\mathcal{H}}  \quad \sum_{j,k} \hat P_{j,k} \mathcal{L}(1, f_k(\x_j^t))+ (1-\hat P_{j,k}) \mathcal{L}(-1,f_k(\x_j^t))+\lambda \sum_k \|f_k\|^2
  \label{eq:learn_classif}
\end{equation}
where $\hat\matP$ is the transported class proportion matrix $\hat\matP=\frac{1}{N_t}\G^\top\matP^s$. Interestingly this formulation illustrates that for each target sample, the data fitting term is a convex sum of hinge loss for a negative and positive label with weights in $\G$.

%% file: expe.tex
 
In this section we evaluate the performance of our method ({\bf JDOT}) on two different transfer tasks of
classification and regression on real datasets~\footnote{Open Source Python implementation of JDOT: \url{https://github.com/rflamary/JDOT}}.


\textbf{Caltech-Office classification dataset.} This dataset~\cite{saenko10} is dedicated to visual adaptation. It contains images 
from four different domains: {\em Amazon}, the {\em Caltech-256} image collection, {\em Webcam}  and {\em DSLR}. Several features, such 
as presence/absence of background, lightning conditions, image quality, etc.) induce a distribution shift between the domains, 
and it is therefore relevant to consider a domain adaptation task to perform the classification. Following~\cite{courty2016}, we choose deep learning 
features to represent the images, extracted as the weights of the fully connected 6th layer of the DECAF convolutional neural
network~\cite{donahue14}, pre-trained on ImageNet. The final feature vector is a sparse 4096 dimensional vector. 

\begin{table}[!t]
	\caption{Accuracy on the Caltech-Office Dataset. Best value in bold.}
	\begin{center}\small
	\label{tab:coffice}
		\begin{tabular}{cccccccc}
		\toprule
			{Domains} & {{\bf Base}} & {{\bf SurK}} & {{\bf SA}} & {{\bf ARTL}} & {{\bf OT-IT}} & {{\bf OT-MM}} & {{\bf JDOT}}\\
			\midrule
			caltech$\rightarrow$amazon & $92.07$ & $91.65$ & $90.50$ & $92.17$ & $89.98$ & $\bf 92.59$ & $91.54$\\
			caltech$\rightarrow$webcam & $76.27$ & $77.97$ & $81.02$ & $80.00$ & $80.34$ & $78.98$ & $\bf 88.81$\\
			caltech$\rightarrow$dslr   & $84.08$ & $82.80$ & $85.99$ & $88.54$ & $78.34$ & $76.43$ & $\bf 89.81$\\
			amazon$\rightarrow$caltech & $84.77$ & $84.95$ & $85.13$ & $85.04$ & $85.93$ & $\bf 87.36$ & $85.22$\\
			amazon$\rightarrow$webcam  & $79.32$ & $81.36$ & $\bf 85.42$ & $79.32$ & $74.24$ & $85.08$ & $84.75$\\
			amazon$\rightarrow$dslr    & $86.62$ & $87.26$ & $\bf 89.17$ & $85.99$ & $77.71$ & $79.62$ & $87.90$\\
			webcam$\rightarrow$caltech & $71.77$ & $71.86$ & $75.78$ & $72.75$ & $\bf 84.06$ & $82.99$ & $82.64$\\
			webcam$\rightarrow$amazon  & $79.44$ & $78.18$ & $81.42$ & $79.85$ & $89.56$ & $90.50$ & $\bf 90.71$\\
			webcam$\rightarrow$dslr    & $96.18$ & $95.54$ & $94.90$ & $\bf 100.00$ & $99.36$ & $99.36$ & $98.09$\\
			dslr$\rightarrow$caltech   & $77.03$ & $76.94$ & $81.75$ & $78.45$ & $\bf 85.57$ & $83.35$ & $84.33$\\
			dslr$\rightarrow$amazon    & $83.19$ & $82.15$ & $83.19$ & $83.82$ & $\bf 90.50$ & $\bf 90.50$ & $88.10$\\
			dslr$\rightarrow$webcam    & $96.27$ & $92.88$ & $88.47$ & $\bf 98.98$ & $96.61$ & $96.61$ & $96.61$\\\hline
			\bf Mean  & $83.92$ & $83.63$ & $85.23$ & $85.41$ & $86.02$ & $86.95$ & $\bf 89.04$\\
			\bf Mean rank & $5.33$ & $5.58$ & $4.00$ & $3.75$ & $3.50$ & $2.83$ & $2.50$\\
			\bf p-value & $<0.01$ & $<0.01$ & $0.01$ & $0.04$ & $0.25$ & $0.86$ & $-$\\
		\bottomrule
		\end{tabular}
	\end{center}
\end{table}

We compare our method with four other methods: the surrogate kernel approach (\cite{Zhang13}, denoted {\bf SurK}), subspace 
adaptation for its simplicity and good performances on visual adaptation (\cite{Fernando13}, {\bf SA}), Adaptation Regularization based Transfer Learning (\cite{LongTKDE14}, {\bf ARTL}), and the two variants of regularized 
optimal transport~\cite{courty2016}: entropy-regularized {\bf OT-IT} and classwise regularization implemented with the Majoration-Minimization 
algorithm~{\bf OT-MM}, that showed to give better results in practice than its group-lasso counterpart. The classification is conducted with a SVM 
together with a linear kernel for every method. Its results when learned on the source domain and tested on the target domain are also reported to serve
as baseline ({\bf Base}). All the methods have hyper-parameters, that are selected using the reverse cross-validation of Zhong and colleagues~\cite{Zhong10}.
The dimension d for {\bf SA} is chosen from $\{1,4,7,\hdots,31\}$. The entropy regularization for 
{\bf OT-IT} and {\bf OT-MM} is taken from $\{10^2,\hdots,10^5\}$, $10^2$ being the minimum value for the Sinkhorn algorithm to prevent numerical errors. 
Finally the $\eta$ parameter of {\bf OT-MM} is selected from $\{1,\hdots,10^5\}$ and the $\alpha$ in {\bf JDOT} from $\{10^{-5},10^{-4},\hdots,1\}$.

The classification accuracy for all the methods is reported in Table~\ref{tab:coffice}. We can see that {\bf JDOT} is consistently outperforming the baseline ($5$ points
in average), indicating that the adaptation is successful in every cases. Its mean accuracy is the best as well as its average ranking. We conducted a Wilcoxon signed-rank test to test if  {\bf JDOT} was statistically better than the other methods, and report the p-value in the tables. This test shows that  {\bf JDOT} is statistically better than the considered methods, except for OT based ones that where state of the art on this dataset \cite{courty2016}.  

\begin{table}[tp]
\caption{Accuracy on the Amazon review experiment. Maximum value in bold font.}
	\label{tab:amazon}\footnotesize\small\vspace{-2mm}
	\begin{center}
		\begin{tabular}{ccccc}
		\toprule
			{Domains} & {{\bf NN}} & {{\bf DANN}} & {{\bf JDOT} (mse)} & {{\bf JDOT} (Hinge)}\\
			\midrule
			books$\rightarrow$dvd               & $0.805$ & $\bf 0.806$ & $0.794$ & $0.795$\\
			books$\rightarrow$kitchen           & $0.768$ & $0.767$ & $0.791$ & $\bf 0.794$\\
			books$\rightarrow$electronics       & $0.746$ & $0.747$ & $0.778$ & $\bf 0.781$\\
			dvd$\rightarrow$books               & $0.725$ & $0.747$ & $0.761$ & $\bf 0.763$\\
			dvd$\rightarrow$kitchen             & $0.760$ & $0.765$ & $0.811$ & $\bf 0.821$\\
			dvd$\rightarrow$electronics         & $0.732$ & $0.738$ & $0.778$ & $\bf 0.788$\\
			kitchen$\rightarrow$books           & $0.704$ & $0.718$ & $\bf 0.732$ & $0.728$\\
			kitchen$\rightarrow$dvd             & $0.723$ & $0.730$ & $0.764$ & $\bf 0.765$\\
			kitchen$\rightarrow$electronics     & $\bf 0.847$ & $0.846$ & $0.844$ & $0.845$\\
			electronics$\rightarrow$books       & $0.713$ & $0.718$ & $0.740$ & $\bf 0.749$\\
			electronics$\rightarrow$dvd         & $0.726$ & $0.726$ & $\bf 0.738$ & $0.737$\\
			electronics$\rightarrow$kitchen     & $0.855$ & $0.850$ & $0.868$ & $\bf 0.872$\\
			\midrule
			\bf Mean  & $0.759$ & $0.763$ & $0.783$ & $\bf 0.787$\\
			\bf p-value  & $0.004$ & $0.006$ & $0.025$ & $-$\\
		\bottomrule
		\end{tabular}
	\end{center}\vspace{-5mm}
\end{table}

\textbf{Amazon review classification dataset} We now consider the {\em Amazon review dataset}~\cite{blitzer2006} which contains online reviews
of different products collected on the Amazon website. Reviews are encoded with bag-of-word unigram and bigram features as input.
The problem is to predict positive (higher than 3 stars) or negative (3 stars or less) notation of reviews (binary classification). 
Since different words are employed to qualify the different categories of products, a domain adaptation task can be formulated  if one 
wants to predict positive reviews of a product from labelled reviews of a different product. Following~\cite{chen2012,ganin2016}, we consider 
only a subset of four different types of product: books, DVDs, electronics and kitchens. This yields 12 possible adaptation tasks. 
Each domain contains 2000 labelled samples and approximately 4000 unlabelled ones. We therefore use these unlabelled samples to 
perform the transfer, and test on the 2000 labelled data.     

The goal of this experiment is to compare to the state-of-the-art method on this subset, namely Domain adversarial neural network (\cite{ganin2016},
denoted {\bf DANN}), and to show the versatility of our method that can adapt to any type of classifier. The neural network used for all methods in this experiment is 
a simple 2-layer model with sigmoid activation function in the hidden layer to promote non-linearity. $50$ neurons are used in this hidden layer.  For 
{\bf DANN}, hyper-parameters are set through the reverse cross-validation proposed in~\cite{ganin2016}, and following the recommendation of authors the 
learning rate is set to $10^{-3}$. In the case of {\bf JDOT}, we used the heuristic setting of $\alpha=1/\max_{i,j}  \dist{\xsi,\xtj}$, and as such we do not need any cross-validation. 
The  squared Euclidean norm is used for both metric in feature space and we test as loss functions both mean squared errors (mse) and Hinge losses. $10$ iterations of the block coordinate descent are realized. For each method, we stop the learning process of the network after $5$ epochs.   
Classification accuracies are presented in table~\ref{tab:amazon}.  The neural network ({\bf NN}), trained on source and tested on target, is also presented 
as a baseline.  {\bf JDOT} surpasses {\bf DANN} in 11 out of 12 tasks (except on books$\rightarrow$dvd). The Hinge loss is better in than mse in 10 out of 12 cases, 
which is expected given the superiority of the Hinge loss on classification tasks~\cite{rosasco2004}.

\textbf{Wifi localization regression dataset}
\begin{table*}[!t]
  \caption{Comparison of different methods on the Wifi localization dataset. Maximum value in bold.}\vspace{-2mm}
  \label{tab:wifi}
  \begin{center}
    \resizebox{\textwidth}{!}{\begin{tabular}{ccccccccc}
    \toprule
      {Domains} & {{\bf KRR}} & {{\bf SurK}} & {{\bf DIP}} & {{\bf DIP-CC}} & {{\bf GeTarS}} & {{\bf CTC}} & {{\bf CTC-TIP}} & {{\bf JDOT}}\\
      \midrule 
      t1 $\rightarrow$ t2 &  80.84$\pm$1.14 &  90.36$\pm$1.22 & 87.98$\pm$2.33 & 91.30$\pm$3.24 & 86.76 $\pm$ 1.91 & 89.36$\pm$1.78 & 89.22$\pm$1.66& {\bf 93.03 $\pm$ 1.24}\\
      t1 $\rightarrow$ t3 &  76.44$\pm$2.66 & {\bf 94.97$\pm$1.29 }& 84.20$\pm$4.29 & 84.32$\pm$4.57 & 90.62$\pm$2.25 & 94.80$\pm$0.87 & 92.60 $\pm$ 4.50& 90.06 $\pm$ 2.01\\
      t2 $\rightarrow$ t3 &  67.12$\pm$1.28 & 85.83 $\pm$ 1.31 & 80.58 $\pm$ 2.10 & 81.22 $\pm$ 4.31 & 82.68 $\pm$ 3.71 & 87.92 $\pm$ 1.87 & {\bf 89.52 $\pm$ 1.14}& 86.76 $\pm$ 1.72\\
      \midrule
      hallway1 & 60.02  $\pm$2.60 & 76.36 $\pm$ 2.44 & 77.48 $\pm$ 2.68 & 76.24$\pm$  5.14& 84.38 $\pm$  1.98& 86.98 $\pm$ 2.02& 86.78 $\pm$ 2.31& {\bf 98.83$\pm$0.58}\\
      hallway2 & 49.38 $\pm$ 2.30 & 64.69  $\pm$0.77 & 78.54 $\pm$ 1.66 &77.8$\pm$  2.70 &   77.38 $\pm$ 2.09 &87.74 $\pm$ 1.89 &87.94 $\pm$ 2.07& {\bf 98.45$\pm$0.67}\\
      hallway3 & 48.42  $\pm$1.32 &  65.73 $\pm$ 1.57 & 75.10$\pm$  3.39 &73.40$\pm$  4.06& 80.64 $\pm$ 1.76 &82.02$\pm$  2.34 &81.72 $\pm$ 2.25& {\bf 99.27$\pm$0.41}\\
    \bottomrule
    \end{tabular}}
  \end{center}\vspace{-5mm}
\end{table*}
For the regression task, we use the cross-domain indoor Wifi localization dataset that was proposed by Zhang and co-authors~\cite{Zhang13},
and recently studied in~\cite{Gong16}. From a multi-dimensional signal (collection of signal strength perceived from several access points), the goal 
is to locate the device in a hallway, discretized into a grid of $119$ squares, by learning a mapping from the signal to the grid element. This 
translates as a regression problem. As the signals were acquired at different time periods by different devices, a shift can be encountered and 
calls for an adaptation. In the remaining, we follow the exact same experimental protocol  as in~\cite{Zhang13,Gong16} for ease of comparison. 
Two cases of adaptation are considered: {\bf transfer across periods}, for which three time periods t1, t2 and t3 are considered, and {\bf transfer 
across devices}, where three different devices are used to collect the signals in the same straight-line hallways  (hallway1-3), leading to three different adaptation 
tasks in both cases. 

We compare the result of our method with several state-of-the-art methods: kernel ridge regression with RBF kernel ({\bf KRR}), surrogate kernel (\cite{Zhang13}, denoted 
{\bf SurK}), domain-invariant projection and its cluster regularized version (\cite{Baktashmotlagh13}, denoted respectively {\bf DIP} and {\bf DIP-CC}), 
generalized target shift (\cite{Zhang2015}, denoted {\bf GeTarS}), and conditional transferable components, with its target information preservation regularization
(\cite{Gong16}, denoted respectively {\bf CTC} and {\bf CTC-TIP}). As in~\cite{Zhang13,Gong16}, the hyper-parameters of the competing methods are 
cross-validated on a small subset of the target domain. In the case of {\bf JDOT}, we simply set the $\alpha$ to the heuristic value of $\alpha=1/\max_{i,j}  \dist{\xsi,\xtj}$ as discussed previously, and $f$ is estimated with kernel ridge regression.  

Following~\cite{Zhang13}, the accuracy is measured in the following way: the prediction is said to be correct if it falls within a range of three meters in the transfer 
across periods, and six meters in the transfer across devices. For each experiment, we randomly sample sixty percent of the source and target domain, and report 
the mean and standard deviation of ten repetitions accuracies in Table~\ref{tab:wifi}. For transfer across periods, {\bf JDOT} performs best in one out of three
tasks. For transfer across devices, the superiority of  {\bf JDOT} is clearly assessed, for it reaches an average score $> 98\%$, which is at least ten points ahead of the best competing
method for every task. Those extremely good results could be explained by the fact that using optimal transport allows to consider large shifts of distribution,
for which divergences (such as maximum mean discrepancy used in {\bf CTC}) or reweighting strategies can not cope with.



%% file: supplementary_full.tex
\section{Illustration on a simple example}

We illustrate the behavior of our method on a  3-class toy example (Figure~\ref{fig:illus}).  We consider a  classification problem using the hinge loss and $\mathcal{H}$ is a Reproducing Kernel Hilbert Space. Source domain samples are drawn from three different 2D Gaussian distributions with with different centers and standard deviations. The target domain is obtained rotating the source distribution by $\pi/4$ radian. Two types of kernel are considered: linear and RBF.   In Figure~\ref{fig:illus}.a, one can observe on the first column of images 
that using directly a classifier learned on the source domain leads to bad performances because of the rotation. We then show the iterations of the block coordinate descent which 
allows one to recover the true labels of the target domain. It is also interesting to examine the impact of the $\alpha$ parameter on the success of the method. In Figure~\ref{fig:illus}.b, we show the evolution of classification 
accuracy for six different $\alpha$ in the case of RBF kernel. 
 Relying mostly on the label cost ($\alpha=\{0.1\}$) leads to a deterioration of the final accuracy. Using only the input space distance ($\alpha=\{50, 100\}$), which is equivalent to \cite{courty2016}, allows a performance gain. But it is clear that using both losses with  $\alpha=\{0.5, 1, 10\}$ leads to the best performance. Also note the small number of iterations required ($<10$) for achieving a steady state.
 
 \begin{figure*}[!hb]
 \centerline{
 \includegraphics[width=0.66\linewidth]{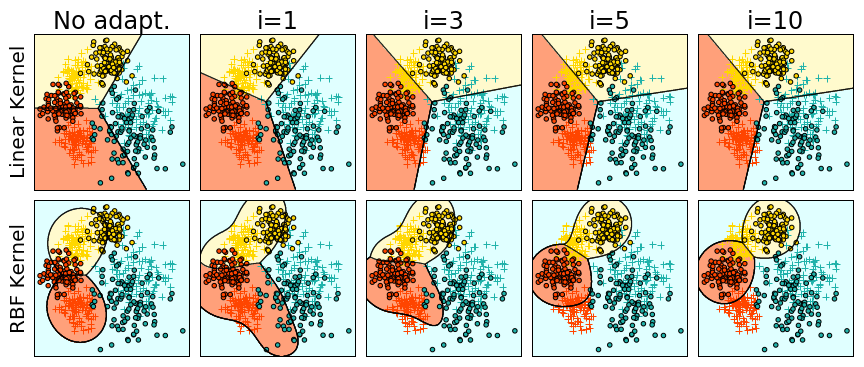}a
\includegraphics[width=0.33\linewidth]{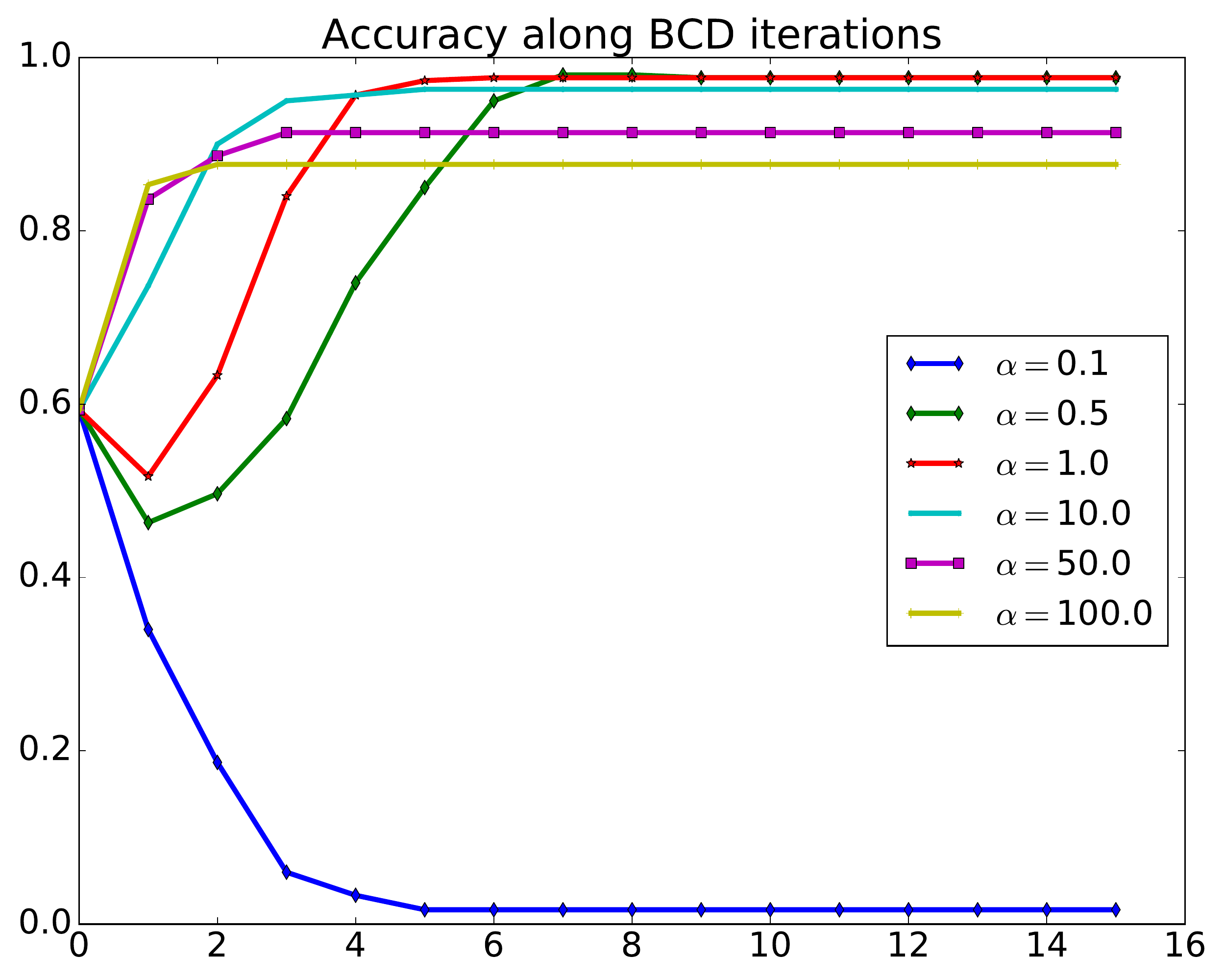}b
}\vspace{-3mm}
\caption{\textbf{Illustration on a toy example}.  (a): Decision boundaries for linear and RBF kernels on selected iterations. The source domain is depicted with crosses, while the target domain samples are class-colored circles. (b): Evolution of the accuracy along 15 iterations of the method for different values of the $\alpha$ parameter;} 
  \label{fig:illus}
\end{figure*}

\section{Block coordinate descent algorithm for solving JDOT}
We give in~algorithm~1 an overview of the block coordinate descent algorithm used for solving JDOT.

\begin{algorithm}
\begin{algorithmic}
\STATE{Initialize function $f^0$ and set $k=1$ }
\STATE Set  $\alpha$ and $\lambda$ 
\WHILE{not converged} 
\STATE{$\G^k\leftarrow$ Solve OT problem (3 in paper) with fixed $f^{k-1}$} 
\STATE{$f^k \leftarrow$ Solve learning problem (7 in paper) with fixed $\G^k$} 
\STATE{$k \leftarrow k +1$}
\ENDWHILE
\end{algorithmic}
\label{algo:bcd}
\caption{Optimization with Block Coordinate Descent}
\end{algorithm}

\section{BCD iterations on real data}

\begin{table}
\centering
\begin{tabular}{|c||c|c|c|}\hline
Iter & caltech $\rightarrow$ amazon & dslr $\rightarrow$   amazon & webcam $\rightarrow$ caltech \\\hline
0 &  89.14 & 80.9 & 75.6 \\
1 &  91.75 & 86.22 & 80.23 \\
2 &  91.44 & 86.95 & 81.75 \\
3 &  91.54 & 87.68 & 82.64 \\
4 &  91.44 & 87.68 & 83.26 \\
5 &  91.65 & 88.0 & 83.35 \\
6 &  91.86 & 88.1 & 83.17 \\
7 &  92.17 & 87.89 & 83.08 \\
8 &  92.28 & 87.58 & 83.26 \\
9 &  92.28 & 87.58 & 83.26 \\
10 &  92.28 & 87.68 & 83.35 \\
11 &  92.28 & 87.79 & 83.44 \\
12 &  92.28 & 87.79 & 83.44 \\
13 &  92.28 & 87.79 & 83.44 \\
14 &  92.28 & 87.79 & 83.44 \\\hline

\end{tabular}
\caption{Accuracy of the estimated model along BCD iterations on Caltech-office dataset}
\label{tab:bcdtable}
\end{table}

We report in Table \ref{tab:bcdtable}, for a fixed set of parameter (no CV), the evolution of the empirical error along the iterations of the 15 first iterations of the BCD on a real dataset. We can see that generally the result stabilizes at around 10 iterations. We can also observe that the increase in performance is not monotonic, contrary to the toy example.

\section{Proof of Theorem 3.1}
We first recall some hypothesis used for this theorem.\\ 

 $\mathcal{H}\subset \mathcal{C}^\Omega$ is the hypothesis class.
$\loss:\mathcal{C}\times\mathcal{C}\rightarrow \mathbb{R}^+$ is the loss function measuring the discrepancy between two labels. This loss is assumed  to be symmetric, bounded and $k$-lipschitz in its second argument, {\it i.e.} there exists $k$ such that for any $y_1, y_2, y_3 \in \mathcal{C}$:\vspace*{-.08cm}
$$
|\loss(y_1,y_2)-\loss(y_1,y_3)|\leq k |y_2-y_3|.\vspace*{-.08cm}
$$

$\pt$ and $\ps$ are respectively the target and source distributions over $\Omega\times\mathcal{C}$, with $\mu_t$ and $\mu_s$  the respective marginals over $\Omega$.
The expected loss in the target domain $err_T(f)$ is defined for any\vspace*{-.08cm} $f\in\mathcal{H}$ as
\begin{equation*}
err_T(f)\equaldef  \E_{(\x,y)\sim\pt} \loss(y,f(\x)). \vspace*{-.08cm}
\end{equation*}
We can similarly define $err_S(f)$ in the source domain and the expected inter function loss $err_T(f,g)=\E_{(\x,y)\sim\pt} \loss(g(\x),f(\x))$.\\

The proxy $\pt^f$ over $\Omega\times\mathcal{C}$ of $\pt$ w.r.t. to $\mu_t$ and $f$ is defined as:
$\pt^f=(\x,f(\x))_{\x\sim \mu_t}$.\\

We consider the following transport loss function:
\begin{eqnarray*}
W_1(\ps,\pt^f)=\!\!&\!\!\displaystyle \inf_{\Pi\in\Pi(\ps,\pt^f)}\int_{(\Omega\times  \mathcal{C})^2} \alpha d(\x_s,\x_t)+\loss(y_s,y_t)\d\Pi((\x_s,y_s),(\x_t,y_t)).\nonumber
\end{eqnarray*}

We now recall the definition of the theorem with all the assumptions.
\begin{theorem} \label{thm:bound}
Let $\mathcal{H}\subset\mathcal{C}^\Omega$ be the hypothesis class where $\Omega$ is a compact mesurable space of finite dimension accompanied with a metric $d$, and $\mathcal{C}$ is the output space. Let $f$ be any labeling function of $\in\mathcal{H}$. Let $\ps$,  $\pt$, $\pt^f$ be three probability distributions over $\Omega\times\mathcal{C}$ with bounded support, with  $\pt^f$ defined w.r.t. the marginal $\mu_t$ of $\pt$ and $f$,    accompanied with a sample of $N_s$ labeled source instances drawn from
$\ps$ and $N_t$ unlabeled instances drawn from $\mu_t$ and labeled by $f$, such that $\ps$ and $\pt^f$ and the associated samples follow the assumptions of Theorem \ref{th:W_1C}. Let $\Pi^*=\argmin_{\Pi\in\Pi(\ps,\pt^f)}\int_{(\Omega\times  \mathcal{C})^2} \alpha d(\x_s,\x_t)+\loss(y_s,y_t)\d\Pi((\x_s,y_s),(\x_t,y_t))$.
Let $f^*$ be a Lipschitz labeling function of $\mathcal{H}$, that verifies the $\phi$-probabilistic transfer Lipschitzness (PTL) assumption with respect to $\Pi^*$ and that minimizes the joint error $err_S(f^*)+err_T(f^*)$ w.r.t all compatible PTL functions with $\Pi^*$. We assume the instance space $\mathcal{X}\subseteq \Omega$ is bounded\footnote{Since the input space is bounded by say a constant $K$: $\|\x\|\leq K$, since $f^*$ is supposed $l$-Lipschitz, then we have for any $\x_1,\x_2$: $|f(\x_1)-f(\x_2)|\leq l\|\x_1-\x_2\|\leq 2lK=M$.} such that $|f^*(\x_1)-f^*(\x_2)|\leq M$ for all $\x_1,\x_2\in \mathcal{X}^2$.  Let $\loss$ be any loss function  symmetric, $k$-lipschitz and that satisfies the triangle inequality.
 Then, there exists, $c'$ and $N$, such that for $N_s>N$ and $N_t>N$, for 
all $\lambda>0$, with $\alpha=k\lambda$, we have with probability at least $1-\delta$:
\begin{eqnarray*}
\err_T(f)\!\!&\!\!\leq\!\!&\!\! \small W_1(\hat{\ps},\hat{\pt^f})+\sqrt{\frac{2}{c'}\log(\frac{2}{\delta})}\left(\frac{1}{\sqrt{N_s}}+\frac{1}{\sqrt{N_t}}\right)\nonumber\\
&&+err_S(f^*)+ err_T(f^*)+kM\phi(\lambda).
\end{eqnarray*}
\end{theorem}
\begin{proof}
\begin{eqnarray}
\err_T(f)&=&\!\!\!\!\!E_{(\x,y)\sim P_t} \loss(y,f(\x))\nonumber\\
&\leq&\!\!\!\!\! E_{(\x,y)\sim P_t} \loss(y,f^*(\x))+\loss(f^*(\x),f(\x))\nonumber\\
&=&\!\!\!\!\!E_{(\x,y)\sim P_t} \loss(f(\x),f^*(\x))+err_T(f^*)\label{eq1:sym}\\
&=&\!\!\!\!\!E_{(\x,y)\sim \pt^f}\loss(f(\x),f^*(\x))+err_T(f^*)\label{eq1:prox}\\
&=&\!\!\!\!\!err_{T^f}(f^*)-err_S(f^*)+err_S(f^*)+err_T(f^*)\nonumber\\
&\leq&\!\!\!\!\!|err_{T^f}(f^*)-err_S(f^*)|+err_S(f^*)+err_T(f^*)\label{eq1:fin}
\end{eqnarray}
Line \eqref{eq1:sym} is due to the symmetry of the loss. Line \eqref{eq1:prox} comes from the fact that: \\ $E_{(\x,y)\sim \pt} L(f(\x),f^*(\x))=E_{(\x,f(\x))\sim \pt^f} L(f(\x),f^*(\x))\equaldef err_{T^f}(f^*(\x))$. \\

Now, we have
\begin{eqnarray}
\lefteqn{\left|err_{T^f}(f^*)-err_S(f^*)\right|}\nonumber\\
&=&\left|\int_{\Omega\times
                 \mathcal{C}}\loss(y,f^*(\x))(\pt^f(\boldsymbol{X}=\x,Y=y)-\ps(\boldsymbol{X}=\x,Y=y))d\x dy\right|\nonumber\\
&=&\left|\int_{\Omega\times        \mathcal{C}}\loss(y,f^*(\x))d(\pt^f-\ps)\right|\nonumber\\
&\leq&\int_{(\Omega\times \mathcal{C})^2} \left|\loss(y_t^f,f^*(\x_t))-\loss(y_s,f^*(\x_s))\right|\d\Pi^*((\x_s,y_s),(\x_t,y_t^f))\label{eq:KTth}\\
&=&\int_{(\Omega\times \mathcal{C})^2} \left|\loss(y_t^f,f^*(\x_t))-\loss(y_t^f,f^*(\x_s))+\right.\nonumber\\
&&\left.\loss(y_t^f,f^*(\x_s))-\loss(y_s,f^*(\x_s))\right|\d\Pi^*((\x_s,y_s),(\x_t,y_t^f))\nonumber\\
&\leq&\int_{(\Omega\times  \mathcal{C})^2} \left|\loss(y_t^f,f^*(\x_t))-\loss(y_t^f,f^*(\x_s))\right|\nonumber\\
&&+\left|\loss(y_t^f,f^*(\x_s))-\loss(y_s,f^*(\x_s))\right|\d\Pi^*((\x_s,y_s),(\x_t,y_t^f))\nonumber\\
&\leq&\int_{(\Omega\times  \mathcal{C})^2} k\left|f^*(\x_t)-f^*(\x_s)\right|+\nonumber\\
&&\left|\loss(y_t^f,f^*(\x_s))-\loss(y_s,f^*(\x_s))\right|\d\Pi^*((\x_s,y_s),(\x_t,y_t^f))\label{eq:k-L}\\
&\leq&k*M*\phi(\lambda)+\int_{(\Omega\times  \mathcal{C})^2} k \lambda d(\x_t,\x_s)+\nonumber\\
&&\left|\loss(y_t^f,f^*(\x_s))-\loss(y_s,f^*(\x_s))\right|\d\Pi^*((\x_s,y_s),(\x_t,y_t^f))\label{eq:PL}\\
&\leq&\int_{(\Omega\times  \mathcal{C})^2} \alpha d(\x_s,\x_t)+\loss(y_t^f,y_s)\d\Pi^*((\x_s,y_s),(\x_t,y_t^f))+k*M*\phi(\lambda)\label{peqsymavfin}\\
&\leq&\int_{(\Omega\times  \mathcal{C})^2} \alpha d(\x_s,\x_t)+\loss(y_s,y_t^f)\d\Pi^*((\x_s,y_s),(\x_t,y_t^f))+k*M*\phi(\lambda)\label{peqfin}\\
&=&W_1(\ps,\pt^f)+k*M*\phi(\lambda).\label{peqfindernier}
\end{eqnarray}
Line \eqref{eq:KTth} is a consequence of the duality form of the Kantorovitch-Rubinstein theorem saying that  for any coupling  $\Pi \in \Pi(P_s,P^f_t)$, we have:
\begin{eqnarray*}
\lefteqn{\left|\int_{\Omega\times        \mathcal{C}}\loss(y,f^*(\x))d(\pt^f-\ps)\right|}\\
&=&\left|\int_{(\Omega\times \mathcal{C})^2} \loss(y_t^f,f^*(\x_t))-\loss(y_s,f^*(\x_s))\d\Pi((\x_s,y_s),(\x_t,y_t^f))\right|\\
&\leq&\int_{(\Omega\times \mathcal{C})^2} \left|\loss(y_t^f,f^*(\x_t))-\loss(y_s,f^*(\x_s))\right|\d\Pi((\x_s,y_s),(\x_t,y_t^f)).
\end{eqnarray*}

Since the inequality is true for any coupling, it is then also true for $\Pi^*$.
Inequality \eqref{eq:k-L} is due to the $k$-lipschitzness of the loss $\loss$ in its second argument. 
Inequality \eqref{eq:PL} uses the fact that $f^*$ and $\Pi^*$ verify the probabilistic transfer Lipschitzness property with probability  $1-\phi(\lambda)$, additionally, taking into account that the deviation between 2 instances with respect to $f^*$ is bounded by $M$ we have the additional term $kM\phi(\lambda)$ that covers the regions where the PTL does not hold.
\eqref{peqsymavfin} is obtained by the symmetry of $d$, the use of triangle inequality on $\loss$ and by replacing $k \lambda$ by $\alpha$. Other inequalities above are due the use of triangle inequality or properties of the absolute value. The last line~\eqref{peqfindernier} is due to the definition of $\Pi^*$.


Now, note that by the use of triangle inequality:
\begin{eqnarray}
W_1(\ps,\pt^f)&\leq&W_1(\ps,\hat{\ps})+W_1(\hat{\ps},\hat{\pt^f})+W_1(\hat{\pt^f},\pt^f)\\
&\leq&W_1(\hat{\ps},\hat{\pt^f})+\sqrt{\frac{2}{c'}\log(\frac{2}{\delta})}\left(\frac{1}{\sqrt{N_s}}+\frac{1}{\sqrt{N_t}}\right).\label{eqC}
\end{eqnarray}
Indeed, the cost function $\mathcal{D}((\x_s,y_s),(\x_t,y_t))=\alpha d(\x_1,\x_2)+\loss(y_1,y_2)$ defines a distance over $(\Omega\times\mathcal{L})^2$,  assuming that $\ps$ and $\pt^f$ have bounded support and the fact that our loss function is bounded,  we can apply Theorem~\ref{th:W_1C} (presented below) 
on $W_1(\ps,\hat{\ps})$ and $W_1(\hat{\pt^f},\pt)$ above (with probability $\delta/2$ each). The two settings may have different constants $N$ and $c'$ and and we consider the maximum $N$ and the minimum $c'$  that comply with both cases.\\

Combining inequalities \eqref{eq1:fin},  \eqref{peqfindernier}, inequality \eqref{eqC} and the use of the union bound, the theorem holds with probability at least $1-\delta$ for any $f\in\mathcal{H}$. $\Box$
\end{proof}

Note that, additionally to the analysis in the paper, a link can be made with classic generalization bounds when the two distributions are equal, {\it i.e.} $\ps=\pt$.
Indeed, if we can choose $f^*$ as the true labeling function on  source/target domains such that $f^*$ is strongly $\phi$-lipschitz w.r.t. $\Pi^*$ ({\it i.e.} $\phi(\lambda)$ is almost 0), then the bound is similar to a classic generalization bound: terms involving $f^*$ are null and using the same sample for source and target $d(\x_1,\x_2)=0$ w.r.t the best alignment. Thus, it remains only the label loss which corresponds to a classic supervised learning loss.

\section{Empirical concentration result for Wasserstein distance}
We give now the result from Bolley and co-authors used in the previous section.
\begin{theorem}[from \cite{Bolley07}, Theorem 1.1.]\label{th:W_1C}
Let $\mu$ be a probability measure in $Z$ so that for some $\alpha>0$ we have for any $\z'$ $\int_{\mathbb{R}^d}e^{\alpha dist(\z,\z')^2}d\mu<\infty$ and $\hat{\mu}=\frac{1}{N}\sum_{i=1}^N\delta_{z_i}$ be the associated empirical measure defined on a sample of independent variables $\{\z_i\}_{i=1}^N$ drawn from $\mu$. Then, for any $d'>dim(Z)$ and $c'<c$, there exists some constant $N_0$ depending on $d'$ and some square exponential moments of $\mu$ such that for any $\epsilon >0$ and $N\geq N_0 \max(\epsilon^{-(d'+2)},1)$,
$$
P[W_1(\mu,\hat{\mu})>\epsilon]\leq \exp\left(-\frac{c'}{2}N\epsilon^2\right)
$$
where $c'$ can be calculated explicitly.
\end{theorem}

Note that $c$ is such that $\mu$ verifies for any measure $\nu$ the Talagrand (transport) inequality  $T_1(c): W_1(\mu,\nu)\leq \sqrt{\frac{2}{c}H(\nu|\mu)}$ with $H$ is the relative entropy. $T_1(c)$ holds when for some $\alpha>0$ and for any $\z'$: $\int_{\mathbb{R}^d}e^{\alpha dist(\z,\z')^2}d\mu(\z)<\infty$, and $c$ can be found explicitly \cite{Bolley07}.
